\documentclass{article}

\usepackage{arxiv}

\usepackage{times}
\usepackage{helvet}
\usepackage{courier}
\usepackage{subfiles}
\usepackage{enumitem}
\usepackage{amsmath}
\usepackage{graphicx}
\usepackage{subcaption}
\usepackage{mwe}
\usepackage[utf8]{inputenc} 
\usepackage[T1]{fontenc}    
\usepackage{hyperref}       
\usepackage{url}            
\usepackage{booktabs}       
\usepackage{amsfonts}       
\usepackage{nicefrac}       
\usepackage{microtype}      
\usepackage{lipsum}
\usepackage{graphicx}
\usepackage{amsthm}
\usepackage{amssymb}
\usepackage{mathtools}
\usepackage{bigints}
\usepackage{xcolor}
\usepackage{chngcntr}
\usepackage{apptools}
\usepackage[section]{placeins}

\usepackage[makeroom]{cancel}

\makeatletter
\newtheoremstyle{indented}
  {3pt}
  {3pt}
  {\addtolength{\@totalleftmargin}{3.5em}
   \addtolength{\linewidth}{-3.5em}
   \parshape 1 3.5em \linewidth}
  {}
  {\bfseries}
  {.}
  {.5em}
  {}
\makeatother

\theoremstyle{plain}
\newtheorem{theorem}{Theorem}
\newtheorem{lemma}{Lemma}
\newtheorem{corollary}{Corollary}
\newtheorem{definition}{Definition}[section]
\newtheorem*{remark}{Remark}

\DeclareMathOperator*{\argmax}{argmax}

\newcommand\figref{Figure~\ref}

\title{Mixtures of Gaussian Processes for regression under multiple prior distributions}

\author{
  Sarem Seitz \\
  Department of Information Systems and Applied Computer Science\\
  Otto-Friedrich-University\\
  Bamberg, Germany\\
  \texttt{sarem.seitz@uni-bamberg.de} \\

}

\begin{document}
\maketitle

\begin{abstract}
When constructing a Bayesian Machine Learning model, we might be faced with multiple different prior distributions and thus are required to properly consider them in a sensible manner in our model. While this situation is reasonably well explored for classical Bayesian Statistics, it appears useful to develop a corresponding method for complex Machine Learning problems. Given their underlying Bayesian framework and their widespread popularity, Gaussian Processes are a good candidate to tackle this task. We therefore extend the idea of Mixture models for Gaussian Process regression in order to work with multiple prior beliefs at once - both a analytical regression formula and a Sparse Variational approach are considered. In addition, we consider the usage of our approach to additionally account for the problem of prior misspecification in functional regression problems.
\end{abstract}


\section{Introduction}
Gaussian Processes (GPs) are a widespread and powerful method in Statistical Learning, allowing to flexibly deal with a large variety of learning tasks. GPs can be interpreted as the Bayesian equivalent to frequentist kernel models and thus are able to deal with epistemic uncertainty by construction. While traditional kernel methods are often outperformed by modern Deep Learning models, our general understanding is much more mature for the former. In the case of GPs, it is possible to derive many useful theoretical properties and guarantees via Linear Algebra and Stochastic Process theory.

One area where Bayesian kernel methods are particularly powerful is the integration of prior knowledge into the learning process. The choice of covariance kernel directly determines the space of functions over which the GP prior is put. However, as shown in \cite{driscoll1973reproducing}, the functions contained in the covariance RKHS do not necessarily match the prior sample paths drawn from a GP with corresponding covariance function, hence some care has to be taken when expressing functional prior knowledge via GPs. 

While it is fairly common to use rather general kernel functions like the Squared Exponential (SE-) kernel which results in a very flexible prior distribution, this approach ignores to some extent the capability to express more specific functional prior knowledge via GPs. This is certainly owed to the fact that narrowing down the space of functions considered to a very specific one, say the space of linear functions, would not take into account the complexity of typical Machine Learning problems which are often governed by highly complex functional relationships. 

Such challenges are particularly concerning when a GP prior distribution places zero probability on the target function, a problem known as \textit{prior misspecification}. In this case, the posterior distribution cannot converge to the true, underlying function either. While, for example, \cite{infinitepriormisspec} show that a misspecified model still converges to a reasonable posterior in the KL-sense, there is obviously no guarantee that the resulting posterior will be of any practical use.  

As a consequence for GP models, there is a trade-off between using highly specific prior distributions and considering rather broad functional priors. While the former can yield better predictive performance with less samples, the risks of prior misspecification increases for less flexible prior distributions. Consider for example a GP prior with Linear kernel function - if the data generating process is linear as well, the resulting posterior can be expected to produce reasonable results on unseen data, even if the training set is rather small. In addition, prediction for new datapoints outside the distribution of the training set will likely yield superior results compared to more flexible priors as long as the linear relation holds over the whole domain of the data. On the other hand, if the underlying function is highly non-linear, a Linear kernel will probably fail to reach the performance of, for example, a Squared Exponential (SE-) kernel. 

For a very simple example, consider the space of all possible functions that map some input domain to the real line. It would be inefficient to place all probability mass on the subset of functions induced by a Linear kernel and zero probability mass unless there is clear evidence that this must be the case. On the other hand, if we scatter our prior probability mass over a larger subspace - say the space induced by the SE-kernel - our inference might be inefficient if the true, inferred function actually lies in the subspace induced by the linear kernel. 

On another note, once we actually consider specific, humanly influenced prior distributions for GPs, we could easily encounter a modeling problem where multiple prior distributions from various sources are available. Hence, the need arises to account for such a set of prior beliefs in a sensible manner. A reasonable solution under such condition would condense the information from all prior distributions provided and at the same time - with reference to the misspecification issue from before - allow for the possibility that all those prior beliefs might still be far off from the function in question. 

Hence, our goal in this paper is to develop a GP framework that allows us to deal with multiple prior views, while at the same time account for the realistic assumption that all those priors are incorrect. We therefore apply considerations from Bayesian prior pooling \cite{combiningexperts} and the idea of a 'catch-all' prior distribution introduced in \cite{baysianopenmindedness}. As it turns out, Mixtures of Gaussian Processes (MGPs) are a reasonable candidate to carry out our considerations in a GP framework. While a closed form regression formula will be shown, we will also consider sparse variational MGPs (SV-MGPs) as described in \cite{hensmanbig} for standard GPs. This allows our idea to work more or less independently from the amount of training data.

\section{Gaussian Processes and sparse approximations}
Before summarizing the fundamentals of GPs and SVGPs, we want to briefly justify our decision to work with GPs instead of other possible options: While we could approach our problem in weight space, for example via Bayesian Neural Networks (BNNs) and Bayesian Linear Models (BLM) with meaningful basis functions, the functional GP route offers a more self-contained and maturely researched methodology. Describing the prior belief expressed by a combination of a BNN and BLM is by far less straightforward than the functional variant using a combination of an SE- and a Linear kernel. 

Formally, a GPs \cite{rasmussengaussprocs} introduce a prior $p(f)$\footnote{We suppress the notion of function input $X$} over functions and a likelihood $p(y|f)$ where interest lies in the posterior distribution of f, $p(f|y)$ which is usually inferred via Bayes' law  

\begin{equation}\label{bayeslaw}
    p(f|y)=\frac{p(y|f)p(f)}{p(y)}.
\end{equation}

The prior distribution is a Gaussian Process specified via

\begin{equation}\label{gaussproc}
    p(f)=\mathcal{GP}(f|m(\cdot),k(\cdot,\cdot'))
\end{equation}

A Gaussian Process is fully defined by its mean function $m(\cdot):\mathcal{X}\mapsto\mathbb{R}$ and covariance kernel function $k(\cdot,\cdot):\mathcal{X}\times\mathcal{X}\mapsto \mathbb{R}_0^+$ where input domain $\mathcal{X}\subseteq\mathbb{R}^k$ the input domain of $f$. It is fairly common to set the prior mean $m(x)=0$ for all $x\in\mathcal{X}$ unless a more suitable mean function is known in advance. A popular choice for $k(\cdot,\cdot)$ is the ARD\footnote{\textbf{A}utomatic \textbf{R}elevance \textbf{D}etection}-kernel

\begin{equation}\label{ardkern}
    k_{ARD}(x,x')=\sigma^2 exp(-0.5 (x-x')\Sigma^{-1}(x-x')))
\end{equation}

where $\Sigma$ is a diagonal matrix with elements $diag(\Sigma)=[l_1^2,...,l_m^2]$. The primary strength of this kernel choice is its ability to naturally work with multi-dimensional input vectors $x$ and evaluate the relevance of each dimension of the input in a principled manner. In case that $\mathcal{X}\subseteq\mathbb{R}$, i.e. the input domain is one-dimensional, \eqref{ardkern} reduces to the well-known SE-kernel from earlier.
 
Usually, Gaussian Processes models are treated through a symmetric, positive semi-definite Gram-Matrix $K$ induced by the kernel function, where $K_{(ij)}=k(x_i,x_j)$.

In case of a Gaussian likelihood $p(y|f)=\mathcal{N}(y|f,\sigma^2)$, it is possible to directly calculate a corresponding posterior distribution for new inputs $X_*$ as

\begin{equation}\label{postpred}
    p(f_*|y)=\mathcal{MVN}(f_*|\Lambda y, K_{**}-\Lambda (K_{nn}+I\sigma^2) \Lambda^T))
\end{equation}

with $\Lambda=K_{*n}(K_{nn}+I\sigma^2)^{-1}$, $K_{*n,(ij)}=k(x^*_i,x_j)$, $K_{**,(ij)}=k(x^*_i,x^*_j)$; $I$ denotes the identity matrix with dimension according to $K_{nn}$. The latter denotes the kernel Gram-Matrix evaluated over all training inputs. We will only consider GPs under a Gaussian likelihood but our approach could be extended to other commonly used likelihoods. In addition, we assume that observations $y$ are sampled i.i.d conditional on $f$, i.e. $p(y|f)=\prod_{i=1}^n p(y_i|f_i)$ for a sample of size $n$. We therefore only need to consider the diagonal of the covariance matrix in \eqref{postpred} when deriving the posterior predictive distribution for $y_*$.\\

The major obstacle for scalable inference with this formulation lies in the quadratic complexity of the Gram-Matrix which makes the standard GP method infeasible for larger datasets. A popular solution to this issue is the usage of Sparse Variational Gaussian Processes (SVGPs) \cite{titsiassvgp, hensmanbig, hensmansvgpc}. In SVGPs, we approximate the true posterior process $p(f|y)$ by a variational process $q(f)$ obtained as

\begin{equation}\label{varproc}
    q(f)=\int_u p(f|u)q(u)du
\end{equation}

where $u$ denotes $m$ inducing variables corresponding to $m$ so-called \textit{inducing locations} $Z\in\mathbb{R}^{k\times m}$ and, primarily for mathematical convenience, $q(u)=\mathcal{MVN}(u|a,LL^T)$, $a\in\mathbb{R}^m$, $L\in \mathbb{R}^{m\times m}$. This setup leads, after applying some calculus for Gaussian distributions, to a tractable functional approximation at any input $X_*$, namely

\begin{equation}\label{varproc}
    q(f)=\mathcal{MVN}(f|\tilde{\Lambda} a, K_{**}-\tilde{\Lambda}(K_{mm}-LL^T)\tilde{\Lambda}^T)
\end{equation}

with $\tilde{\Lambda}=K_{*l}K_{ll}^{-1}$, $K_{*l,(ij)}=k(x^*_i,z_j)$, $K_{mm,(ij)}=k(z_i,z_j)$. Contrary to the plain GP formulation, we now have to find optimal $a,L,Z$ by maximizing an evidence lower bound $ELBO$ with respect to those parameters:

\begin{equation}\label{standardelbo}
    ELBO = \sum_{i=1}^n\mathbb{E}_{q(f)}\left[p(y_i|f_i)\right]-KL(q(u)||p(u))
\end{equation}

where $KL(\cdot||\cdot)$ denotes the Kullback Leibler (KL) divergence and we again assume that the $y_i$ are independent from each other given $f_i$. The distribution of $p(u)$ is obtained by evaluating the GP prior distribution at inducing locations $Z$.

Using a variational sparse approach, it is possible to considerably reduce the complexity and allow for applications large big data problems as well. Also \eqref{standardelbo} exemplifies that we are dealing with a trade-off between the variational posterior maximizing the likelihood and having the variational distribution as similar as possible to the prior distribution in the KL-sense.

\section{Prior pooling}
Let us get back to the situation outlined in the beginning where we were considering multiple functional prior beliefs that we can properly express as GP prior distributions $p(f_1),...,p(f_k)$. Those prior beliefs might be derived from other existing models or studies or could be elicited from experts and can range from complex functional expressions to simple heuristics. Keeping in mind the application of Bayes' law for GPs as in \eqref{bayeslaw}, we require some aggregation mapping $p(f_1),...,p(f_k)\mapsto p(f)$ where we simply denote by $p(f)$ a condensed prior distribution that summarizes the information from the individual priors. 

This obviously begs the question of how such mapping could look like in order for the resulting aggregate prior to make sense. As it turns out, there already exists a reasonable amount of literature on that question in classic Bayesian statistics under the name \textit{prior pooling}. For a deeper overview on the topic we refer to \cite{logpooling, linearpooling, combiningexperts, linearpools}. 

A common and intuitive way to aggregate a set of prior beliefs $p_1(f),...,p_k(f)$ into a single one is known as \textit{linear pooling} where the aggregate prior probability distribution is a weighted average of the individual ones:

\begin{equation}\label{linearpooling}
    p(f)=\sum_{i=1}^k \pi_i p_i(f),\quad \sum_{i=1}^k \pi_i = 1.
\end{equation}

Apparently, this expression denotes a mixture probability distribution as, in Machine Learning, is commonly seen in clustering methods, namely \textit{Gaussian Mixture Models} (\cite{gaussianmixtures}). While there exist other options, the linear pooling variant possesses, on the one hand, meaningful theoretical properties (\cite{linearpools}) and, on the other hand turns out to work well with Gaussian Process models. 

Considering the existing literature on mixture models with GPs, we should explicitly keep in mind that \eqref{linearpooling} is derived from a weighted average of probability distributions. This is in contrast to the commonly seen usage of mixture models where the generating model is assumed to have a hierarchical structure with a multinomial prior over the component distributions:

\begin{equation}\label{hierarchicalmix}
    \tilde{p}(f)=\int_{\Pi}p(f|\pi)p(\pi) d\pi,\quad p(\pi)=Multinom(\pi|\pi_1,...,\pi_k)
\end{equation}

The distinction is crucial when considering posterior inference under prior distributions for mixture coefficients $\pi_i$.

\section{Regression via Mixtures of Gaussian Processes}
Since \eqref{linearpooling} describes a mixture distribution and as Gaussian type Mixture Models are well established in Machine Learning, it appears reasonable to consider a Mixture of Gaussian Processes in order to achieve prior aggregation for functional prior distributions.

Recall that a Mixture of $k$ Normal distributions has the following density:

\begin{equation}\label{gaussianmixture}
    p(x)=\mathcal{MXN}(x|\pi_1,...,\pi_k;\mu_1,...,\mu_k;\Sigma_1,...,\Sigma_k)=\sum_{i=1}^k\pi_i \mathcal{N}(x|\mu_i,\Sigma_i).
\end{equation}

Assuming that the reader is familiar with the general ideas of stochastic processes and using \eqref{gaussianmixture} we can easily derive a formal definition of an MGP:\newline

\begin{definition}[Mixture of Gaussian Processes]
    A continuous Stochastic Process $f:\mathcal{X}\mapsto \mathbb{R}$ is a \textit{k-Mixture of Gaussian Processes},
    
    \begin{equation}\label{mgpdef}
        f\sim\mathcal{MGP}(\pi_1,...,\pi_k;m_1(\cdot),...,m_k(\cdot);k_1(\cdot,\cdot),...,k_k(\cdot,\cdot))
    \end{equation}
    
    if its finite dimensional marginals at an index set $x_1,...,x_n\in\mathcal{X}$ are distributed as a Mixture of $k$ Multivariate Normal distributions, i.e.
    
    \begin{equation}\label{mgpmarginals}
        p(f_{x_1},...,f_{x_n})=\sum_{i=1}^k \pi_i \mathcal{N}(f_{x_1},...,f_{x_n}|\mu_i,\Sigma_i)
    \end{equation}
    
    with means $\mu_i$ and covariances $\Sigma_i$ constructed via the respective mean and kernel functions, $m_i(\cdot),k_i(\cdot,\cdot)$ and $\sum_{i=1}^k \pi_i=1$.\newline
    
\end{definition}

It would obviously be ideal if we could directly transfer known formulas from plain GP models to the mixture variant. As we will show in the next subsection, all relevant properties can indeed intuitively be extended to this variant.

\subsection{General Results}\label{generalresults}

The key consideration for this subsection is the interplay between the linearity of the integral and the general formula for an arbitrary mixture distribution $p_{M}(x)$:

\begin{equation}\label{mixturenintegral}
    \int p_{M}(x)dx=\int \sum_{i=1}^k \pi_i p(x)dx = \sum_{i=1}^k \pi_i \int p(x) dx
\end{equation}

Namely, we can decompose an integration over a mixture distribution into weighted sums of integrations of its components. We will now use \eqref{mixturenintegral} and the \textit{Kolmogorov extension theorem} to prove the existence of our proposed mixture stochastic process. This rather simple result is primarily included for completeness' sake as we haven't encountered it yet, despite MGPs being quite regularly seen nowadays. We first state here the extension theorem without a prove (see for example \cite{oksendalbook, bremaud2020probability}):\newline

\begin{theorem}[Kolmogorov extension theorem]
    For all $x_1,...,x_m\in\mathcal{X}, n\in\mathbb{N}$ let $\nu_{x_1,...,x_m}$ be a probability measure on $(\mathbb{R}^n)^m$ subject to the following conditions: 
    \begin{enumerate}
        \item $\nu_{x_{\sigma(1)},...,x_{\sigma(m)}}(B_{\sigma(1)}\times\cdots\times B_{\sigma(m)})=\nu_{x_{1},...,x_{m}}(B_{1}\times\cdots\times B_{m})$ for all permutations $\sigma$ on $\{1,2,...,m\}$ and
    
        \item $\nu_{{x_1},...,{x_m}}(B_1\times\cdots\times B_m)=\nu_{{x_1},...,{x_m},{x_{m+1}},...,x_{m+k}}(B_1\times\cdots B_m\times \mathbb{R}^n\times\cdots\times\mathbb{R}^n)$ for all $k\in\mathbb{N}$ where the right hand side has a total of $m+k$ factors. 
    \end{enumerate}
    Then there exists a probability space $(\Omega, \mathcal{F},P)$ and a stochastic process $\{f_x\}$ on $\Omega, f_x:\Omega\mapsto\mathbb{R}^n$ subject to
    
    $$\nu_{x_i,...,x_k}(B_1\times\cdots B_m)=P[f_{x_1}\in B_1, \cdots, f_{x_k}\in B_k]$$
    
    for all $x_i\in\mathcal{X},m\in\mathbb{N}$ and Borel sets $B_i$.\newline
\end{theorem}

\begin{remark}
    In our case, we obviously have $n=1$.\newline
\end{remark}

\begin{corollary}
    Let $f$ denote a Mixture of $k$ Gaussian Processes $f^{(1)},...,f^{(k)}$ and denote by $\nu^{f}_{x_1,...,x_m}$ the corresponding probability measure induced by $p(f_{x_1},...,f_{x_m})$ with respect to Lebesgue-Measure. Then, $\nu^{f}_{x_1,...,x_m}$ fulfills both conditions for the Kolmogorov extension theorem. \newline
\end{corollary}

A proof to this and all upcoming statements can be found in the appendix. It is obvious that such a process can only be unique up to swapping of its mixture components as for example discussed in \cite{mixgpidentifiable} for clustering with Mixtures of Gaussians.

Regarding actual applications though, it is by far more interesting to obtain results on conditional and marginal distributions of MGPs. As it turns out the commonly marginalization and conditioning properties for Mulitvariate Normal distributions hold naturally in the mixture case:\newline

\begin{lemma}\label{sigmaplus}
    Let $X=[X_A^T,X_B^T]^T$ denote a vector of random variables, obtained by stacking random variable vectors $X_A,X_B$ whose joint distribution is a Mixture of $k$ Multivariate Normal distributions
    
    \begin{equation}\label{mixturenormalstack}
        p(x)=p\left(\begin{bmatrix}x_A \\ x_B\end{bmatrix}\right)=\sum_{i=1}^k \pi_i \mathcal{N}\left(\begin{bmatrix}x_A \\ x_B\end{bmatrix}\bigg\vert\begin{bmatrix}\mu_{i,A} \\ \mu_{i,B}\end{bmatrix}, \begin{bmatrix}
                                                         \Sigma_{i,A} & \Sigma_{i,AB}^T  \\
                                                         \Sigma_{i,AB} & \Sigma_{i,B}  \\
                                                        \end{bmatrix}\right).
    \end{equation}
    
    Then, 
    \begin{enumerate}
        \item the \textbf{marginal distribution} of $X_A$ is a Mixture of $k$ Multivariate Normal distributions with density
    
            \begin{equation}\label{mixturenormalmarginal}
                p(x_A)=\sum_{i=1}^k \pi_i \mathcal{N}(x_A|\mu_{i,A}, \Sigma_{i,A}),
            \end{equation}
            
        \item the \textbf{conditional distribution} of $X_A$ given $X_B$ is a Mixture of $k$ Multivariate Normal distributions with density
    
            \begin{equation}\label{mixturenormalconditional}
                p(x_A|x_B)=\sum_{i=1}^k \frac{\pi_i \mathcal{N}(x_B|\mu_{i,B},\Sigma_{i,B})}{\sum_{j=1}^k \pi_j \mathcal{N}(x_B|\mu_{j,B},\Sigma_{j,B})}\mathcal{N}(x_A|\mu_{i,A|B},\Sigma_{i,A|B}).
            \end{equation}
            
            with 
            
            $$\mu_{i,A|B}= \mu_{i,A}+\Sigma_{i,AB}\Sigma_{i,B}^{-1}(x_B - \mu_{i,B})$$
            $$\Sigma_{i,A|B}=\Sigma_{i,A}-\Sigma_{i,AB}\Sigma_{i,B}^{-1}\Sigma_{i,AB}^T$$\newline
    \end{enumerate}

\end{lemma}

A crucial observation regarding the first part is that a univariate marginal distribution of a Mixture of (multivariate) Gaussians equals a mixture of the corresponding univariate marginals of the components. This is particularly useful when we assume i.i.d observations as is the next result: \newline

\begin{lemma}
    Let $X\sim \mathcal{MXN}(\pi_1,...,\pi_k;\mu_1,...,\mu_k;\Sigma_1,...,\Sigma_k), Y\sim\mathcal{N}(0,I\sigma^2)$ and let 
    
    $$Z=X + Y.$$
    
    Then it follows that
    
    $$Z\sim \mathcal{MXN}(\pi_1,...,\pi_k;\mu_1,...,\mu_k;\Sigma_1+I\sigma^2,...,\Sigma_k+I\sigma^2)$$\newline
\end{lemma}

\subsection{Plain MGP regression}\label{naivemgp}

Using the results from \ref{generalresults}, we can set up our MGP regression model as follows:

\begin{equation}\label{mgp_model}
    y=f(x)+\epsilon,
\end{equation}
$$f\sim\mathcal{MGP}(\pi_1,...,\pi_k;m_1(\cdot),...,m_k(\cdot);k_1(\cdot,\cdot),...,k_k(\cdot,\cdot)),\quad \epsilon\sim\mathcal{N}(0,\sigma^2)$$

It is straightforward to see from $\ref{sigmaplus}$ that integrating out $f$ in \eqref{mgp_model} results in (supressing $x$ from our notation for clarity) 

\begin{equation}\label{intoutf}
    p(y) = \mathcal{MGP}(y|\pi_1,...,\pi_k;m_1(\cdot),...,m_k(\cdot); k_1(\cdot,\cdot)+\sigma^2,...,k_k(\cdot,\cdot)+\sigma^2)
\end{equation}

In conjunction with plain GP regression we can establish the following distributional relation between training inputs and outputs $X,y$ and function evaluations $f_*$ at new inputs $X_*$:

\begin{equation}\label{mgpdistrelation}
    p(y,f_*|X,X_*) = \sum_{i=1}^k \pi_i \mathcal{N}\left(\begin{bmatrix}y \\ f_* \end{bmatrix}\bigg\vert\begin{bmatrix}m_i(X) \\ m_i(X_*)\end{bmatrix}, \begin{bmatrix}
                                                         K_{i,nn}+I\sigma^2 & K_{i,*n}  \\
                                                         K_{i,n*} & K_{i,**}  \\
                                                        \end{bmatrix}\right)
\end{equation}

with $K_{i,nn}$ the $i$-th kernel Gram-matrix evaluated at $X$, $K_{i,**}$ the corresponding Gram-matrix at $X_*$ and $K_{i,n*}$ the cross-evaluation at $(X,X^*)$; $K_{i,*n}=K_{i,n*}^T$ and $m_i(\cdot)$ element-wise evaluation of the mean function. From \eqref{mixturenormalconditional} we can then derive the conditional Mixture of Gaussians:

\begin{equation}\label{conditionalfuncs}
    p(f_*|X,y,X_*) = \sum_{i=1}^k \frac{\pi_i \mathcal{N}(y|m_i(X),K_{i,nn}+I\sigma^2)}{\sum_{j=1}^k \pi_j \mathcal{N}(y|m_j(X),K_{j,nn}+I\sigma^2)}\mathcal{N}(f|m_{i|y}(X_*)_{|y},K_{i|y,**})
\end{equation}

where 

\begin{equation}\label{condfuncmean}
    m_{i|y}(X_*)=m_i(X_*) + K_{i,*n}(K_{i,nn}+I\sigma^2)^{-1}(y-m_i(X))
\end{equation}

\begin{equation}\label{condfunccov}
    K_{i|y,**}=K_{i,**}-K_{i,*n}(K_{i,nn}+I\sigma^2)^{-1}K_{i,n*}
\end{equation}

and finally to obtain predictions for $p(y^*|X,y,X^*)$, we simply need to add the observation variance $I \sigma^2$ to $K_{i|y,**}$ in accordance with Lemma 2.

The hyperparameters of the mean and kernel functions and the observation variance can be learned as usual, either by maximizing the marginal model log-likelihood given some training data or via full Bayesian inference. The Maximium Likelihood objective can then be seen to be

\begin{equation}\label{maxlikeobj}
    \sigma^*;\theta^*=\argmax_{\sigma;\theta} p(y|X)= \argmax_{\pi_1,...,\pi_k,\sigma;\theta} \sum_{i=1}^k \pi_i \mathcal{N}(y|m_i(X),K_{i,nn}+I\sigma^2).
\end{equation}

$$s.t. \sigma>0,\quad \theta\in\Theta$$

where $\theta$ denotes additional hyperparameters and $\Theta$ the set admissible values for $\theta$. 

For the mixture coefficients $\pi_i$, three different options arise:

\begin{enumerate}
    \item \textbf{Fix the $\pi_i$ ex-ante based on confidence in the prior distributions}: This option is commonly seen in the literature about linear prior pooling. Provided that confidence in the prior distributions is uniform across all components, $\pi_i=\pi=\frac{1}{k}$ would be a reasonable choice. On the other hand, assigning a large piece of probability mass to a 'poor' prior distribution could also negatively affect model performance. The corresponding objective is then the Maximum Likelihood case as in \eqref{maxlikeobj}.
    \item \textbf{Learn the $\pi_i$ together with the hyperparameters using standard Maximum Likelihood}: While this method seems naturally appealing from a Machine Learning, data-driven perspective, we more or less lose the ability to express individual confidence in each prior distribution. This would contradict our goal for this paper and hence, we will not consider it any further.
    \item \textbf{Place a prior distribution over the $\pi_i$, for example a Dirichlet distribution}: Using a Dirichlet-like prior distribution for the mixture coefficients is fairly popular in standard Gaussian Mixture models, see for example \cite{infinitedirichlets}. Contrary to the standard case however, we don't assign the Dirichlet prior on the probability parameters of a latent Multinomial distribution but directly on the mixture weights. We will consider this approach in the SV-MGP model later on.
\end{enumerate}

Since we don't assume a latent multinomial selection process as is the case for standard mixture models, we don't apply the usual EM-algorithm for Maximum Likelihood optimization. Instead, we apply plain gradient descent optimization which, as we will see, yields reasonable empirical results. 

So far, we have included the mean functions $m_i(\cdot)$ in all our derivations. It is however common practice to set $m_i(x)=0$ for all $x\in\mathcal{X}$ which makes some calculations less complex. The flexibility of the kernel function is usually sufficient to express a rich range of possible posterior functions. Nevertheless, we should account for additional information about the mean function when they are available.

\subsection{Sparse variational MGP regression}

The most obvious flaw that the naive regression formula in \eqref{conditionalfuncs} has in common its GP counterpart is the incapability to deal with larger datasets. Provided we were using the plain implementation, performance would be even worse than usual as we now need to invert $k$ Gram-matrices. This is certainly an issue for large-scale problems that needs to be solved in order for our method to be applicable for problems beyond 'toy-problem' sized datasets. Inspired by the existing literature about sparse GPs, we now derive a sparse variational variant of MGPs, SV-MGPs.
Ideally we want to obtain a similar optimization objective for as in \eqref{standardelbo} which turns out to be feasible.

We start by putting a Dirichlet prior over the mixture coefficients:

\begin{equation}\label{dircoeff}
    p(\pi_1,...,\pi_k)=p(\pi) = \mathcal{D}ir(\pi|\alpha)
\end{equation}

with $\alpha$ the parameter vector $\alpha_1,...,\alpha_k$. The usage of such prior begs the additional advantage of allowing to express uncertainty about the GP priors in a fully Bayesian manner. 
In order to show this, we will use the following result:\newline

\begin{lemma}
    Let $p(x)=\mathcal{MXN}(x|\pi_1,...,\pi_k;\mu_1,...,\mu_k;\Sigma_1,...,\Sigma_k)$. Then there exists a decomposition
    
    \begin{equation}\label{mixdecomp}
        p(x)=\int_{z_1,...,z_k}p(x|z_1,...,z_k)\prod_{i=1}^k \mathcal{N}(z_i|\mu_i,\Sigma_i)dz_1\cdots dz_k.
    \end{equation}
    
\end{lemma}

While we can derive a distribution for $p(x|z_1,...,z_k)$, it turns out that we only require \eqref{mixdecomp} to be a valid decomposition in order to derive an ELBO for SV-MGPs. In the manner of the original derivation of SVGPs in \cite{titsiassvgp}, we augment the original conditional distribution 

\begin{equation}\label{postpreddec}
    p(y_*|y)=\int p(y_*|f,\pi_1,...,\pi_k)p(f,\pi_1,...,\pi_k|y)df d\pi_1\cdots d\pi_k
\end{equation}

where we now treat the mixture coefficients $\pi$ as Bayesian posterior distributions. Hence, we assign a Dirichlet prior distribution as in \eqref{dircoeff} and a proper variational distribution. We now introduce additional GP function outputs $f_1,...,f_k$ and, additionally, corresponding inducing variables $f_{1_m},...,f_{k_m}$ - each at potentially different inducing locations. To prevent our notation from bloating, we summarize $f_1,...,f_k$ as $f_Z$, $f_{1_m},...,f_{k_m}$ as $f_{Z_m}$ and $\pi=\pi_1,...,\pi_k$ and accordingly, $df_Z = df_{1}\cdots df_{k},df_{Z_m} = df_{1_m}\cdots df_{k_m},d\pi=d\pi_1\cdots d\pi_k$. This leads to the following decomposition of the posterior distribution $p(f,\pi|y)$:

\begin{equation}\label{posteriordec}
    p(f,\pi|y) = \int p(f|f_Z,\pi)\cdot p(f_Z|f_{Z_m})p(f_{Z_m})\cdot p(\pi)\,df_Z df_{Z_m}
\end{equation}

Our interest now lies in finding a variational approximation of $p(f,f_Z,f_{Z_m},\pi|y)$, i.e. we aim to minimize the KL-divergence between posterior and its approximation: 

\begin{equation}\label{klsvmgp}
    KL\left(q(f,f_Z,f_{Z_m},\pi)||p(f,f_Z,f_{Z_m},\pi|y)\right)=\int q(f,f_Z,f_{Z_m},\pi)\log\frac{q(f,f_Z,f_{Z_m},\pi)}{p(f,f_Z,f_{Z_m},\pi|y)}df df_Z df_{Z_m}d\pi.
\end{equation}

Consider the following decomposition for the variational distribution:

\begin{equation}\label{variationaldecomp}
    q(f,f_Z,f_{Z_m},\pi)=p(f|f_Z,\pi)p(f_Z|f_{Z_m})q(f_{Z_m})q(\pi)
\end{equation}

where $q(f_{Z_m})=\prod_{i=1}^k q(f_{Z_i})$ is a set of $k$ independent sets of inducing variables of dimension $m$ and each element jointly Gaussian distributed as

$$q(f_{i_m})=\mathcal{N}(f_{i_m}|m_{f_i},S_{f_i}),\quad m_{f_i}\in\mathbb{R}^m, S_{f_i}=L_{f_i}^T L_{f_i}, L_{f_i}\in\mathbb{R}^{m\times m}$$

and $q(\pi)=\mathcal{D}ir(\pi|\tilde{\alpha})$ is a variational Dirichlet distribution, hence both prior and variational distribution for the mixture coefficients are Dirichlet. The decomposition of the prior distribution follows the same principle as its variational approximation: 

\begin{equation}\label{likelidecomp}
    p(y,f,f_Z,f_{Z_m},\pi)=p(y|f,\pi)p(f|f_Z,\pi)p(f_Z|f_{Z_m})p(f_{Z_m})p(\pi)
\end{equation}

where $p(f_{Z_m})=\prod_{i=1}^k p(f_{Z_i})$ the set of evaluations of $f_{1},...,f_{k}$ at inducing locations $X_{i,m}$

$$p(f_{i_m})=\mathcal{N}(f_{i_m}|m_{i,m}),K_{i,mm})$$

with $m_{i,m},K_{i,mm}$ the $i$-th mean and kernel functions evaluated over $X_{i,m}$. As an alternative, we could evaluate $f_{1},...,f_{k}$ over the same set of inducing locations. During our experiments however, we found this setup to produce better results though. 

Let, similar to \cite{titsiassvgp},

\begin{equation}
    p(f_i|f_{i_m}))=\mathcal{N}(f_i|K_{i,nm} K_{i,mm}^{-1}f_{i_m},K_{i,nn} - K_{i,nm} K_{i,mm}^{-1}K_{i,mn})
\end{equation}

which, given an arbitrary multivariate Gaussian with of feasible dimension $\mathcal{N}(f_{i_m}|m,S)$, results in

\begin{equation}\label{vargpmarg}
    \int p(f_i|f_{i_m}))\mathcal{N}(f_{i_m}|m,S) df_{i_m}=\mathcal{N}(f_i|K_{i,nm} K_{i,mm}^{-1}m,K_{i,nn} - K_{i,nm} K_{i,mm}^{-1}(K_{i,mm}-S)K_{i,mm}^{-1}K_{i,mn})
\end{equation}

where $K_{i,nn}$ denotes the $k_i(\cdot,\cdot)$ evaluated over inputs $X$, $K_{i,nm}$ the cross evaluation over $X$ and inducing inputs $X_{i,mm}$ and $K_{i,mn}=K_{i,nm}^T$. Finally, we have 

\begin{equation}
    \int p(f|f_Z,\pi)\mathcal{N}(f_Z|\mu_Z, \Sigma_Z)\mathcal{D}ir(\pi|\alpha)df_Z d\pi = MXN(f|\tilde{\pi},\mu_Z,\Sigma_Z)
\end{equation}

where $\tilde{\pi}_i=\frac{\alpha_i}{\sum_{j=1}^k \alpha_j}$ and we summarized $\mathcal{N}(f_Z|\mu_Z, \Sigma_Z)=\prod_{i=1}^k\mathcal{N}(f_i|\mu_i, \Sigma_i)$; $\mu_Z=\mu_1,...,\mu_k; df_Z=df_1\cdots df_k$ and $\Sigma_Z$ accordingly. Such $p(f|f_Z)$ exists according to \textbf{Lemma 3}.

In addition, these results imply the following:

\begin{equation}\label{variationalsvmgp}
    \begin{gathered}
            q(f)=\int p(f|f_Z)p(f_Z|f_{Z_m})p(f_{Z_m})q(\pi) df_Z df_{Z_m}d\pi\\
             =\int p(f|f_Z,\pi)\prod_{i=1}^k\mathcal{N}(f_i|K_{i,nm} K_{i,mm}^{-1}m_{f_i},K_{i,nn} - K_{i,nm} K_{i,mm}^{-1}(K_{i,mm}-S_{f_i})K_{i,mm}^{-1}K_{i,mn}) \mathcal{D}ir(\pi|\tilde{\alpha}) df_Z d\pi\\
             = MXN(\tilde{\pi}_1,...,\tilde{\pi}_k;\tilde{\mu}_1,...,\tilde{\mu}_k;\tilde{\Sigma}_1,...,\tilde{\Sigma}_k)
    \end{gathered}
\end{equation}

with $\tilde{\pi}_i=\frac{\alpha_i}{\sum_{i=1}^k \alpha-i},\tilde{\mu}_i=K_{i,nm} K_{i,mm}^{-1}m_{f_i},\tilde{\Sigma}_i=K_{i,nn} - K_{i,nm} K_{i,mm}^{-1}(K_{i,mm}-S_{f_i})K_{i,mm}^{-1}K_{i,mn})$.\\

As usual, we cannot minimize the KL-divergence in \eqref{klsvmgp} directly but rather need to maximize an ELBO derived thereof. Putting all this together, we obtain the following ELBO for SV-MGPs (a derivation can be found in Appendix \ref{svmgpelbo}):

\begin{equation}\label{svmgpelbo}
    \begin{gathered}
        ELBO =\sum_{j=1}^n\left\{\sum_{i=1}^k\pi_i\mathcal{N}(y_j|k_{i,j}^T K_{i,mm}^{-1}m_{f_i},\sigma^2)-\frac{1}{2\sigma^2}\tilde{k}_{i,(j,j)}
            -\frac{1}{2}tr\left(S_{f_i}\Lambda_{i,j}\right)\right\} \\
            -\sum_{i=1}^k KL(q(f_{i_m})||p(f_{i_m})) -KL(q(\pi)||p(\pi))
    \end{gathered}
\end{equation}

where $KL(q(\pi)||p(\pi))$ is the KL-divergence between two Dirichlet distributions $\mathcal{D}ir(\pi|\alpha),\mathcal{D}ir(\pi|\tilde{\alpha})$ whose closed form representation we'll also state in the appendix. Hence, the whole ELBO in \eqref{svmgpelbo} is available in closed form and we don't need to rely on any approximations. Predictions for new inputs $X_*$ can also be calculated analytically:

\begin{equation}\label{postpred}
    p(y_*) = \int p(y_*|f_*)q(f_*) df_* = \mathcal{MXN}(y_*|\tilde{\pi}_1,...,\tilde{\pi}_k;\tilde{\mu}_{1,*},...,\tilde{\mu}_{k,*};\tilde{\Sigma}_{1,*}+I\sigma^2,...,\tilde{\Sigma}_{k,*}+I\sigma^2)
\end{equation}

with $\tilde{\mu}_{i,*},\tilde{\Sigma}_{i,*}$ the evaluation at $X_*$ in correspondence to \eqref{variationalsvmgp}

\section{Related work and context}
Mixtures of Gaussian Processes have been considered in particular in conjunction with Mixtures of local Experts models \cite{originalmoes, twentymoes, surveymoes}. Contrary to our approach, the mixture coefficients $\pi_i$ in local expert models are usually varying as a function of the input space $\mathcal{X}$. Hence, the goal of mixtures of local GP experts is primarily the achievement of superior performance and/or faster computation in comparison to standard GPs \cite{mixgps,deepstrucmgps}. In that regard, there also exist approaches to put a Dirichlet Process prior on the mixture coefficients which results in infinite dimensional mixtures of GPs \cite{infinitedirichlets}. 

Another natural consideration is the treatment of multimodal functional data as in \cite{multimodalmgps, enrichedmixtures,nonparammixtures}. In that case, the mixture distribution is to be interpreted as a direct result of mixing a multinomially distributed selection of function outputs rather than a weighted average of prior densities or measures respectively. Under the multimodality assumption, we would require more than one of the GPs to fit the data which would manifest itself as the mixture coefficients not converging to one. Under our setup, if one of the prior GPs matches the underlying data generating process fairly well, we actually expect the coefficients to converge to one for that particular GP.

Regarding our variational inference approach, the work done by \cite{overlappingstuff} probably comes closest although their underlying assumption is again the multimodality or overlap of multiple functions as mentioned in \eqref{hierarchicalmix} rather than the question of how to aggregate multiple functional prior distributions. This assumption makes the authors use a multinomial prior over the selection of processes resulting in a variational approximation procedure different from ours.

\section{Practical considerations and experiments}
Besides the aggregation of multiple prior distributions that encode the knowledge or information from multiple sources, we want to stress the following use-case scenario of our method:

While our prior knowledge might describe the functional relations in a prediction task fairly well, there exists no guarantee that prior assumptions are completely wrong. As an illustrative example, consider the functional relation between two variables $x$ and $y$ to follow a simple quadratic formula plus random Gaussian noise:

\begin{equation}
    y=x^2 + \epsilon,\quad \epsilon\sim\mathcal{N}(0,\sigma^2).
\end{equation}

An expert might draw the erroneous conclusion that the target $y$ has a linear relationship with $x$, 

\begin{equation}
    y=x + \epsilon,\quad \epsilon\sim\mathcal{N}(0,\sigma^2),
\end{equation}

which would appear to be correct in particular if our training data contains only a few samples samples with realizations of $x$ closely around zero. If all prior probability mass lies on a hypothetical set of linear functions it is, in accordance with Bayesian asymptotics, impossible for the posterior distribution to correctly depict the underlying quadratic relationship. 

On the other hand, if there exists a strong indication for an actual linear relationship, this situation should be reflected in the chosen prior distribution. In order to nevertheless account for the possibility of a mis-specification we propose to include a very broad, functional GP prior distribution - expressed for example through a Squared Exponential Kernel - in the pool of prior distributions. This reflects an idea proposed in \cite{baysianopenmindedness} coined \textit{catch-all hypothesis}, i.e. the introduction of an alternative hypothesis that behaves orthogonally to the union of all other hypotheses. 

\begin{figure}
\centering
\begin{subfigure}{.5\textwidth}
  \centering
  \includegraphics[width=.95\linewidth]{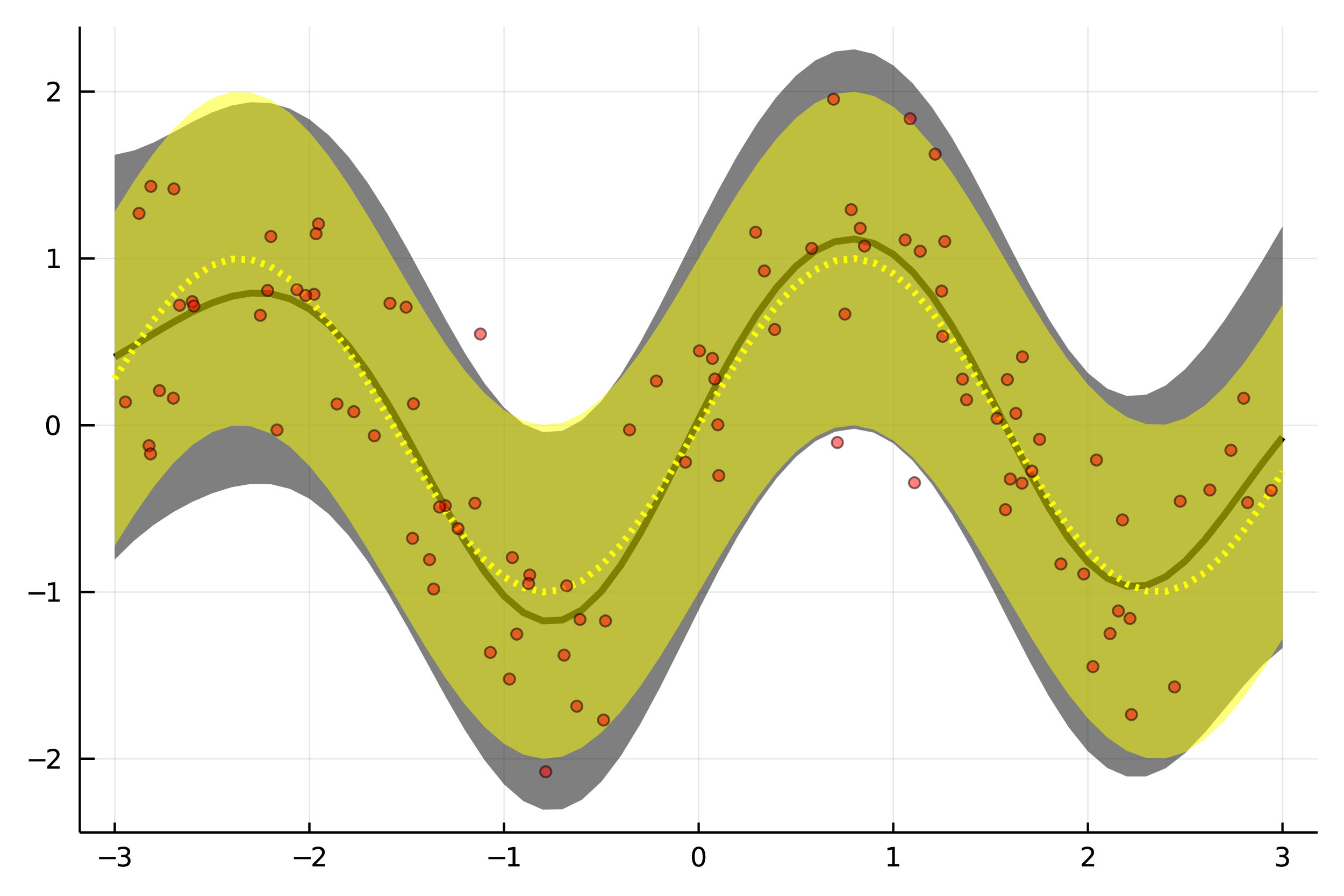}
  \caption{MGP using a mixture of a linear kernel and an SE-kernel\footnotemark}
\end{subfigure}%
\begin{subfigure}{.5\textwidth}
  \centering
  \includegraphics[width=.95\linewidth]{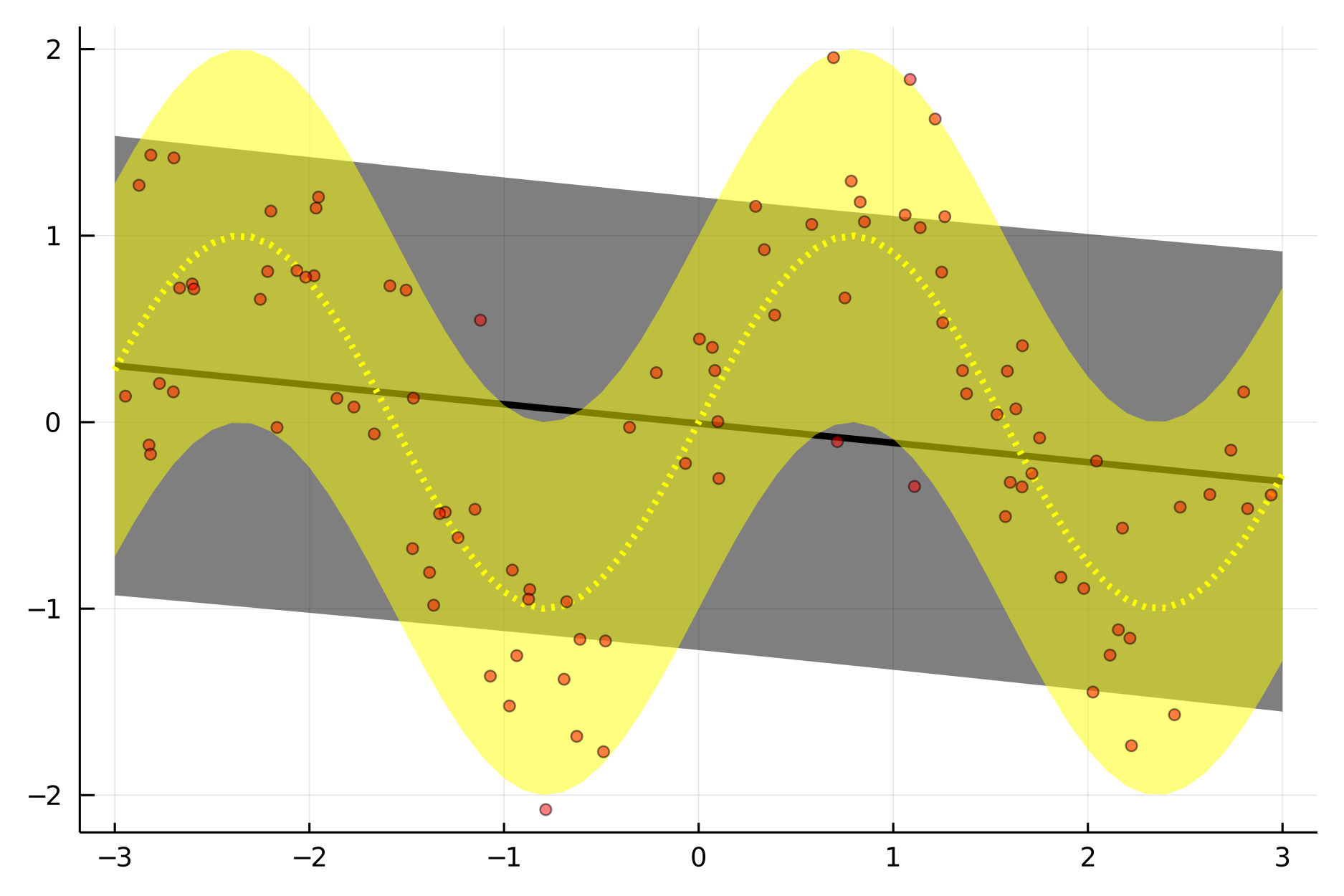}
  \caption{The same data using only a Linear Kernel}
\end{subfigure}
\caption{Plain MGP as a remedy for prior misspecification on data sampled from $y=sin(2x)+\epsilon$ - yellow line and area denote mean and 2 standard deviations of the data generating process, black line and area denote mean and 2 standard deviations of the posterior predictive process given datapoints (red dots)}. 
\label{fig:priormis}
\end{figure}

Since finding an actual orthogonal complement to our set of functional hypotheses poses an additional challenge, we won't do so here but rather use the Square Exponential Kernel to express a 'complementary' functional prior distribution. As can be seen in \figref{fig:priormis}, mixing a misspecified prior - the presumed expert knowledge - with a fallback SE-kernel results in the model employing the SE-kernel to mitigate the misfit of the Linear kernel. This might particulary be helpful when the model is trained incrementally in an online learning problem. While a specific prior distribution might be helpful at the beginning of training, the fallback kernel could take over if the data turns out to not behave accordingly.

\begin{figure}
\centering
\begin{subfigure}{.5\textwidth}
  \centering
  \includegraphics[width=.95\linewidth]{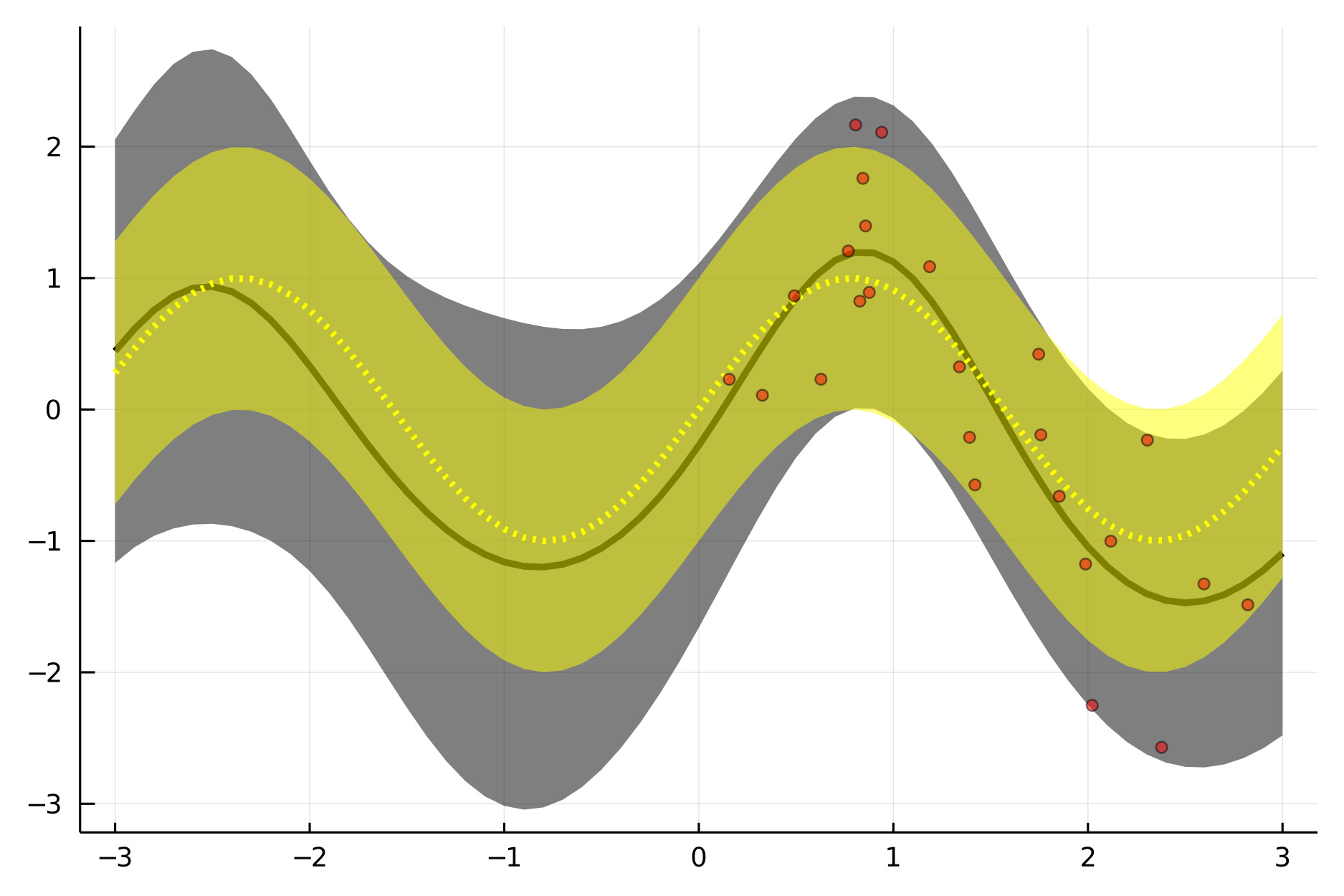}
  \caption{MGP using a mixture of a Periodic kernel and an SE-kernel\footnotemark}
\end{subfigure}%
\begin{subfigure}{.5\textwidth}
  \centering
  \includegraphics[width=.95\linewidth]{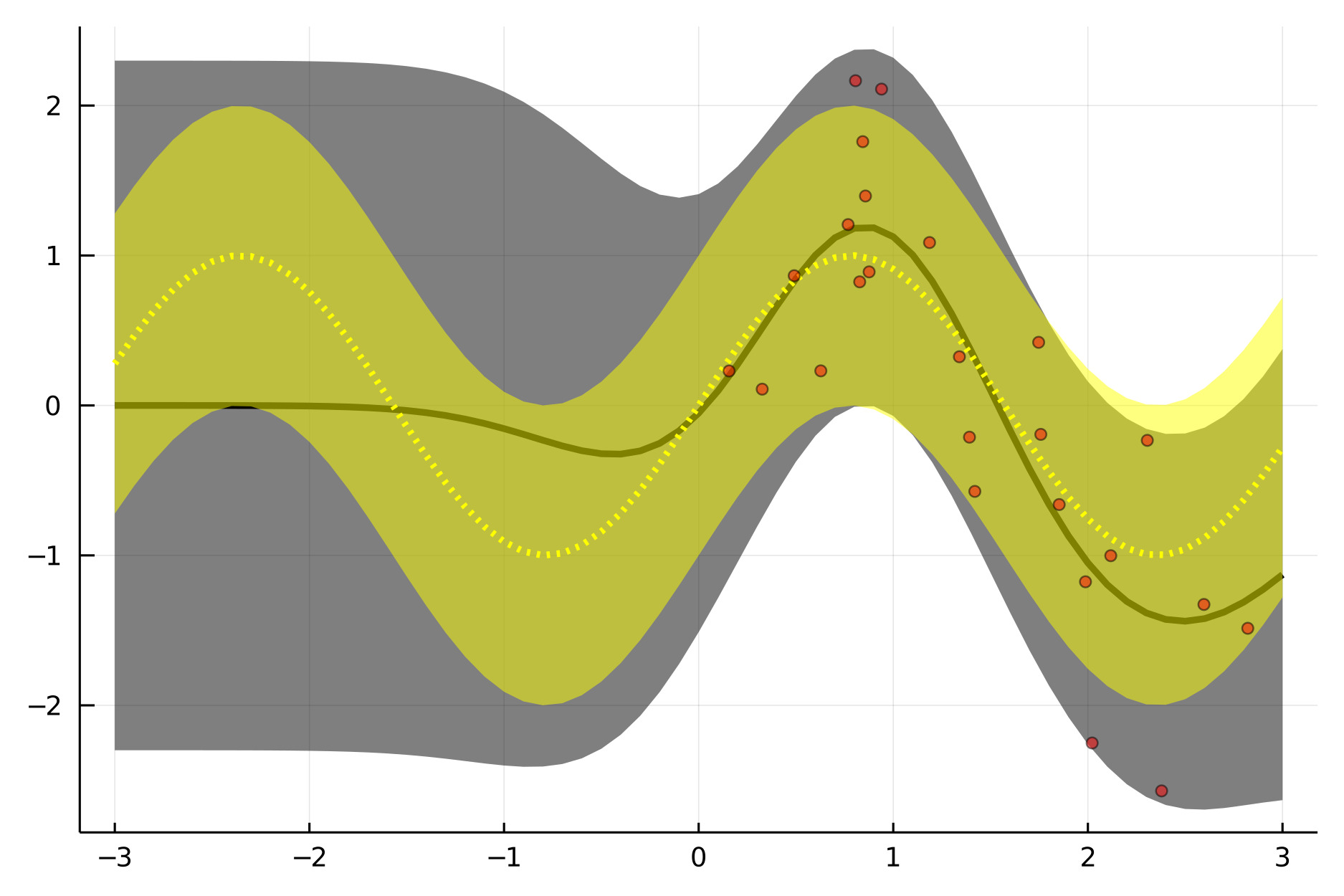}
  \caption{The same data using only an SE-kernel}
\end{subfigure}
\caption{Plain MGP for a small dataset sampled from $y=sin(2x)+\epsilon$ - yellow line and area denote mean and 2 standard deviations of the data generating process, black line and area denote mean and 2 standard deviations of the posterior predictive process given datapoints (red dots) 
}

\label{fig:smallmix}
\end{figure}

\begin{figure}
\centering
  \includegraphics[width=.5\linewidth]{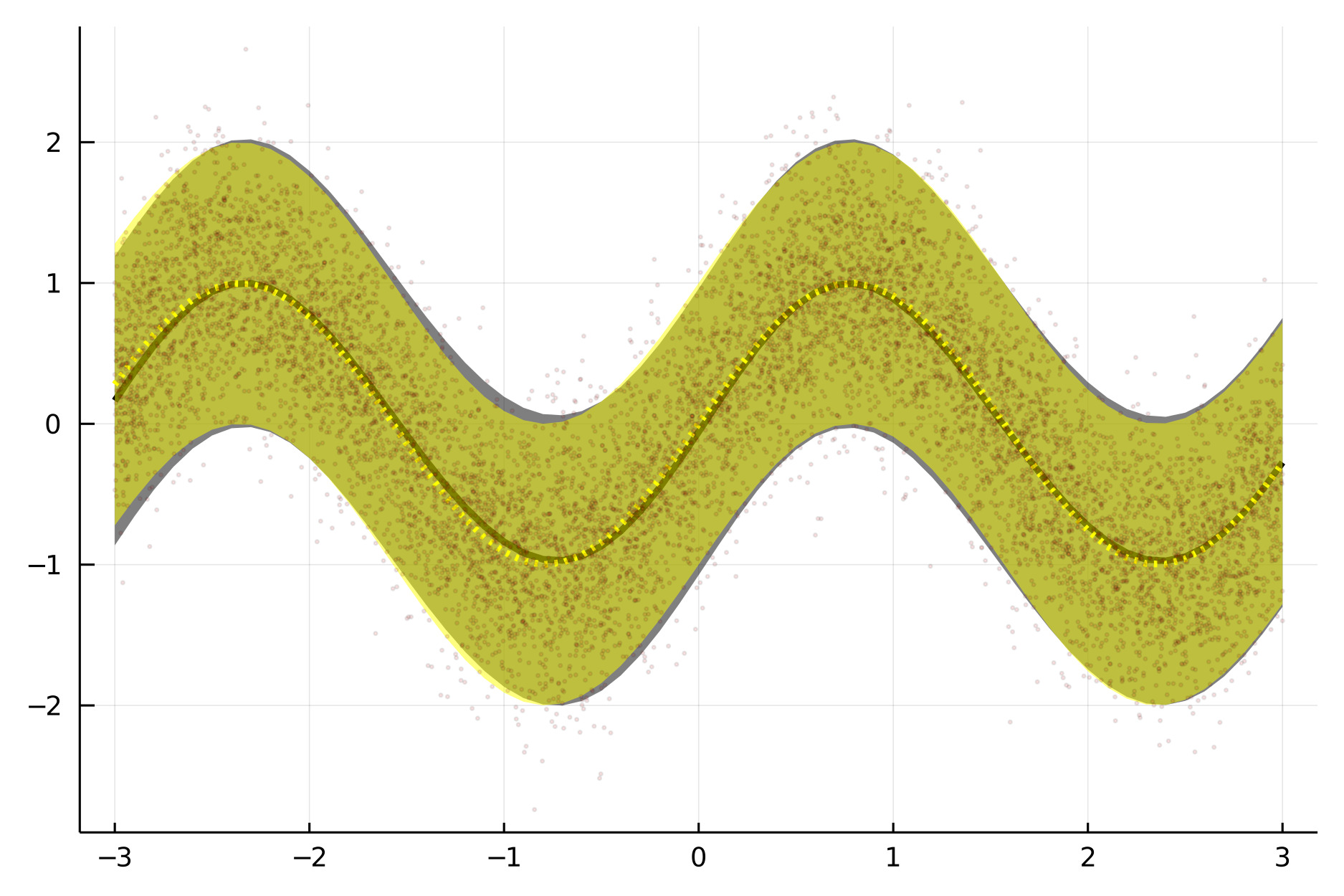}
  \label{fig:svmgp}
\caption{SVMGP posterior process for $N=10,000$ samples from $y=sin(2x)+\epsilon$ - yellow line and area denote mean and 2 standard deviations of the data generating process, black line and area denote mean and 2 standard deviations of the posterior predictive process given datapoints (red dots) 
}

\label{fig:svmgp}
\end{figure}

\footnotetext{both prior GPs were assigned a prior mixture coefficient of 0.5}
\footnotetext{we used the following prior weights: \textbf{Periodic kernel} - 0.8, \textbf{SE -kernel} - 0.2}

In that regard, \figref{fig:smallmix} shows how the 'correct' choice of a specific prior distribution - here, a Periodic kernel for periodic data - could potentially help with the problem of sparsely available data in some regions of the input domain. Here, our model is able to use the prior information provided by the periodic kernel in order to provide sensible predictions about the unobserved, negative, section of the x-axis. This is in contrast to the overly general prior distribution implied by the SE-kernel which is unable to correctly capture the periodicity inherent in the data.

Finally we tested the SV-MGP version discussed before by sampling 10,000 datapoints from $y=sin(2x)+\epsilon$. We used three GPs (linear kernel, periodic kernel, SE-kernel) as mixture components and put a $Dir(3,3,3)$ prior on the mixture coefficients and used 3 inducing points per process. The corresponding (variational) posterior weights are $(0.0006,0.996,0.0034)$. As can be seen in \figref{fig:svmgp}, the resulting posterior predictive process roughly matches the data generating process, i.e. our model empirically converged to ground truth a sufficiently sized dataset. 

Obviously, the experiments in this section are all conducted on a simulated toy dataset. Since we were primarily interested in the theoretical implications of our method, namely mixing multiple prior distributions and adding a fallback prior, we restrained from more sophisticated experiments. One apparent next step is conducting more realistic experiments on real-world datasets.

\section{Discussion}
We have considered MGPs as a potential approach to pool multiple GP prior distributions from different sources in a sound manner. For posterior inference, we have developed both an analytical and a variational solution and demonstrated their applicability on simulated datasets. Extending the simulation to more realistic settings and data would be an interesting next step to consider in future work. 

As a concrete case that we would like to explore is the real-world applicability of our method for computer vision tasks, in particular for situations where training data is only sparsely available. Using GP prior distributions that encode meaningful expert or heuristic knowledge, we could test if our method is able to perform reasonably well when datasets are too small for complex Deep Learning architectures to be superior. This might be an interesting alternative to transfer learning for models in small data regimes, especially when there is no suitable transfer model available.

This is obviously just a hypothesis for now and we can expect the challenge to outperform Deep Learning solutions for learning under small datasets to be highly difficult with GPs. Nevertheless, the integration of meaningful prior knowledge into complex Machine Learning tasks is still a worthwhile topic as a sensible supplement to the mostly data-only learning of current state-of-the-art approaches.


\vskip 0.2in
\bibliographystyle{unsrt}  
\bibliography{references}

\newpage
\begin{appendix}
\setcounter{lemma}{0}
\setcounter{theorem}{0}
\setcounter{corollary}{0}
\section{Proofs}
\begin{corollary}
    Let $f$ denote a Mixture of $k$ Gaussian Processes $f^{(1)},...,f^{(k)}$ and denote by $\nu^{f}_{x_1,...,x_m}$ the corresponding probability measure induced by $p(f_{x_1},...,f_{x_m})$ with respect to Lebesgue-Measure. Then, $\nu^{f}_{x_1,...,x_m}$ fulfills both conditions for the Kolmogorov extension theorem. 
\end{corollary}
    
\begin{proof}

    By the Radon-Nikodym theorem, we have 
    
    $$\nu^{f}_{x_1,...,x_m}(B_1\times\cdots\times B_m)=\int_{B_1\times\cdots\times B_m}p(f_{x_1},...,f_{x_m})df_{x_1}\cdots df_{x_m}$$
    $$=\int_{B_1\times\cdots\times B_m}\sum_{i=1}^k \pi_i p_i(f_{x_1},...,f_{x_m})df_{x_1}\cdots df_{x_m}=\sum_{i=1}^k \pi_i \int_{B_1\times\cdots\times B_m} p_i(f_{x_1},...,f_{x_m})df_{x_1}\cdots df_{x_m}$$
    $$=\sum_{i=1}^k \pi_i \nu^i_{x_1,...,x_m}(B_1\times\cdots\times B_m)$$
    
    where $\nu^i_{x_1,...,x_m}$ denotes the measure induced by the $i$-th mixture component's density with respect to the Lebesgue-Measure. Since the last term depends on $x_1,...,x_m$ and $B_1\times\cdots\times B_m$ respectively solely via $\nu^i_{x_1,...,x_m}$, it suffices to show that the conditions for the Kolmogorov extension theorem are fulfilled by all $\nu^i_{x_1,...,x_m}$. This is indeed the case for Gaussian Processes (see for example \cite{bremaud2020probability} for a formal proof) and therefore the claim follows.\newline
\end{proof}

\textcolor{white}{asdfasdf}\newline

\begin{lemma}
    Let $X=[X_A^T,X_B^T]^T$ denote a vector of random variables, obtained by stacking random variable vectors $X_A,X_B$ whose joint distribution is a Mixture of $k$ Multivariate Gaussians
    
    \begin{equation}\label{mixturenormalstack}
        p(x)=p\left(\begin{bmatrix}x_A \\ x_B\end{bmatrix}\right)=\sum_{i=1}^k \pi_i \mathcal{N}\left(\begin{bmatrix}x_A \\ x_B\end{bmatrix}\bigg\vert\begin{bmatrix}\mu_{i,A} \\ \mu_{i,B}\end{bmatrix}, \begin{bmatrix}
                                                         \Sigma_{i,A} & \Sigma_{i,AB}^T  \\
                                                         \Sigma_{i,AB} & \Sigma_{i,B}  \\
                                                        \end{bmatrix}\right).
    \end{equation}
    
    Then, 
    \begin{enumerate}
        \item the \textbf{marginal distribution} of $X_A$ is a Mixture of $k$ Multivariate Gaussians with density
    
            \begin{equation}\label{mixturenormalmarginal}
                p(x_A)=\sum_{i=1}^k \pi_i \mathcal{N}(x_A|\mu_{i,A}, \Sigma_{i,A}),
            \end{equation}
            
        \item the \textbf{conditional distribution} of $X_A$ given $X_B$ is a Mixture of $k$ Multivariate Gaussians with density
    
            \begin{equation}\label{mixturenormalconditional}
                p(x_A|x_B)=\sum_{i=1}^k \frac{\pi_i \mathcal{N}(x_B|\mu_{i,B},\Sigma_{i,B})}{\sum_{j=1}^k \pi_j \mathcal{N}(x_B|\mu_{j,B},\Sigma_{j,B})}\mathcal{N}(x_A|\mu_{i,A|B},\Sigma_{i,A|B}).
            \end{equation}
            
            with 
            
            $$\mu_{i,A|B}= \mu_{i,A}+\Sigma_{i,AB}\Sigma_{i,B}^{-1}(x_B - \mu_{i,B})$$
            $$\Sigma_{i,A|B}=\Sigma_{i,A}-\Sigma_{i,AB}\Sigma_{i,B}^{-1}\Sigma_{i,AB}^T$$
    \end{enumerate}

\end{lemma}

\begin{proof}
    \textcolor{white}{asfadf}
    \begin{enumerate}
        \item Denote by $\int_B p\left(\begin{bmatrix}x_A \\ x_B\end{bmatrix}\right) dB$ the marginalization over all elements of $x_B$. Then via \eqref{mixturenintegral} we marginalize out $x_B$ in \eqref{mixturenormalstack} as

        $$\int_B p\left(\begin{bmatrix}x_A \\ x_B\end{bmatrix}\right) dB = \int_B \sum_{i=1}^k \pi_i \mathcal{N}\left(\begin{bmatrix}x_A \\ x_B\end{bmatrix}\bigg\vert\begin{bmatrix}\mu_{i,A} \\ \mu_{i,B}\end{bmatrix}, \begin{bmatrix}
                                                     \Sigma_{i,A} & \Sigma_{i,AB}^T  \\
                                                     \Sigma_{i,AB} & \Sigma_{i,B}  \\
                                                    \end{bmatrix}\right) dB$$
        
        $$=\sum_{i=1}^k \pi_i \int_B \mathcal{N}\left(\begin{bmatrix}x_A \\ x_B\end{bmatrix}\bigg\vert\begin{bmatrix}\mu_{i,A} \\ \mu_{i,B}\end{bmatrix}, \begin{bmatrix}
         \Sigma_{i,A} & \Sigma_{i,AB}^T  \\
         \Sigma_{i,AB} & \Sigma_{i,B}  \\
        \end{bmatrix}\right) dB$$
    
    which reduces the problem to marginalizing '$B$' out of $k$ multivariate Normal distributions. The result then follows from the marginalization properties of Gaussian variables.
    
    \item We have
    
        $$p(x)=\sum_{i=1}^k \pi_i \mathcal{N}(x|\mu_{i}, \Sigma_{i})=\sum_{i=1}^k \pi_i \mathcal{N}(x_A|\mu_{i,A|B}, \Sigma_{i,A|B})\mathcal{N}(x_B|\mu_{i,B},\Sigma_{i,B})\,\, \left(=p\left(\begin{bmatrix}x_A \\ x_B\end{bmatrix}\right)\right)$$
        
        Hence
        
        $$p(x_A|x_B)=\frac{p\left(\begin{bmatrix}x_A \\ x_B\end{bmatrix}\right)}{p(x_B)}=\frac{\sum_{i=1}^k \pi_i \mathcal{N}(x_A|\mu_{i,A|B}, \Sigma_{i,A|B})\mathcal{N}(x_B|\mu_{i,B},\Sigma_{i,B})}{p(x_B)}$$
        
        $$=\sum_{i=1}^k\frac{ \pi_i \mathcal{N}(x_B|\mu_{i,B},\Sigma_{i,B})}{p(x_B)}\mathcal{N}(x_A|\mu_{i,A|B}, \Sigma_{i,A|B})$$
        $$=\sum_{i=1}^k \frac{\pi_i \mathcal{N}(x_B|\mu_{i,B},\Sigma_{i,B})}{\sum_{j=1}^k \pi_j \mathcal{N}(x_B|\mu_{j,B},\Sigma_{j,B})}\mathcal{N}(x_A|\mu_{i,A|B},\Sigma_{i,A|B}).$$\newline

    \end{enumerate}
    
    Where the conditional Gaussian parameters for the rightmost factor are calculated as usual.

\end{proof}

\textcolor{white}{asdfasdf}\newline

\begin{lemma}
    Let $X\sim \mathcal{MXN}(\pi_1,...,\pi_k;\mu_1,...,\mu_k;\Sigma_1,...,\Sigma_k), Y\sim\mathcal{N}(0,I\sigma^2)$ and let 
    
    $$Z=X + Y.$$
    
    Then it follows that
    
    $$Z\sim \mathcal{MXN}(\pi_1,...,\pi_k;\mu_1,...,\mu_k;\Sigma_1+I\sigma^2,...,\Sigma_k+I\sigma^2)$$
\end{lemma}

\begin{proof}
    We show this via the Moment Generating Function (MGF) of $X$, $M_X(t)$, with $t\in\mathbb{R}^k$:
    
    $$M_Z(t)=\mathbb{E}[e^{t^T(X+Y)}]=\int \int e^{t^T(x+y)}p(x)p(y)dx dy$$
    $$=\int\int e^{t^Tx}\sum_i \pi_i p_i(x)dx\,e^{t^Ty} p(y)dy=\sum_i \pi_i \int e^{t^Ty}p(y)dy\int e^{t^Tx} p_i(x)dx=\sum_i \pi_i M_{Y}(t)M_{i,X}(t),$$
    $$=\sum_{i=1}^{k}\pi_i e^{t^T I\sigma^2 t}e^{(t^T\mu_i + t^T\Sigma_i t)}=\sum_{i=1}^{k}\pi_i e^{(t^T\mu_i + t^T(\Sigma_i+I\sigma^2) t)}.$$
    
    where the final term describes the MGF of a Mixture of Normal distributions with the prescribed parameters. As a probability distribution is uniquely defined by its MGF, the assertion follows.
    
\end{proof}

\begin{lemma}
    Let $p(x)=\mathcal{MXN}(x|\pi_1,...,\pi_k;\mu_1,...,\mu_k;\Sigma_1,...,\Sigma_k)$. Then there exists a decomposition
    
    \begin{equation}\label{mixdecomp}
        p(x)=\int_{z_1,...,z_k}p(x|z_1,...,z_k)\prod_{i=1}^k p(z_i)dz_1\cdots dz_k
    \end{equation}
    
    such that $p(z_i)=\mathcal{N}(z_i|\mu_i,\Sigma_i)$.\newline
\end{lemma}

\begin{proof}
    Setting 
    
    $$p(x|z_1,...,z_k)=\sum_{i=1}^k \pi \delta_{z_i}(x)$$
    
    with $\delta_{z_i}(x)=\delta(x-z_i)$ the Dirac-delta distribution at $z_i$, results in the claim being true. This can be shown as follows:
    
    $$p(x)=\int_{z_1,...,z_k}\sum_{i=1}^k \pi \delta_{z_i}(x)\prod_{j=1}^k \mathcal{N}(z_j|\mu_j,\Sigma_j)dz_1\cdots dz_k$$
    $$=\sum_{i=1}^k \pi \int_{z_1,...,z_k}\delta_{z_i}(x)\prod_{j=1}^k \mathcal{N}(z_j|\mu_j,\Sigma_j)dz_1\cdots dz_k$$
    $$\stackrel{*}{=} \sum_{i=1}^k \pi_i \mathcal{N}(x|\mu_i,\Sigma_i)=\mathcal{MXN}(x|\pi_1,...,\pi_k;\mu_1,...,\mu_k;\Sigma_1,...,\Sigma_k)$$
    \newline
    
    $(^*)$ This can be seen by noticing that
    $$\int_{z_1,...,z_k}\delta_{z_i}(x)\prod_{j=1}^k\mathcal{N}(z_j|\mu_j,\Sigma_j)dz_1\cdots dz_k$$
    $$\int_{z_i}\delta_{z_i}(x)\mathcal{N}(z_i|\mu_i,\Sigma_i)dz_i\prod_{j=1\setminus i}^k \int_{z_j}\mathcal{N}(z_j|\mu_j,\Sigma_j)dz_j$$
    $$=\mathcal{N}(x|\mu_i,\Sigma_i)\cdot 1^{k-1}=\mathcal{N}(x|\mu_i,\Sigma_i)$$
    
    and 
    $$\int_{z_i}\delta_{z_i}(x)\mathcal{N}(z_i|\mu_i,\Sigma_i)dz_i$$ $$=\int_{z_i}\delta(x-z_i)\mathcal{N}(z_i|\mu_i,\Sigma_i)dz_i$$ $$=\int_{z_i}\delta(z_i-x)\mathcal{N}(z_i|\mu_i,\Sigma_i)dz_i$$
    $$=\mathcal{N}(x|\mu_i,\Sigma_i)$$
    
    follows from the symmetry property of the Dirac-delta distribution, $\delta(x)=\delta(-x)$
    
\end{proof}
\newpage
\section{Derivation of the SV-MGP ELBO}\label{svmgpelbo}
We derive this result by starting with the KL-divergence as in \eqref{klsvmgp}:

$$\int q(f,f_Z,f_{Z_m},\pi)\log\frac{q(f,f_Z,f_{Z_m},\pi)}{p(f,f_Z,f_{Z_m},\pi|y)}df df_Z df_{Z_m} d\pi$$

$$=\int q(f,f_Z,f_{Z_m},\pi)\log\frac{q(f,f_Z,f_{Z_m},\pi)p(y)}{p(y,f,f_Z,f_{Z_m},\pi)}df df_Z df_{Z_m} d\pi$$

$$\stackrel{\eqref{likelidecomp}}{=}\int q(f,f_Z,f_{Z_m},\pi)\log\frac{q(f,f_Z,f_{Z_m},\pi)p(y)}{p(y|f)p(f|f_Z,\pi)p(f|f_Z)p(f_Z|f_{Z_m}) p(\pi)}  df df_Z df_{Z_m} d\pi$$

$$\stackrel{\eqref{variationaldecomp}}{=}\int p(f|f_Z,\pi)p(f_Z|f_{Z_m})q(f_{Z_m})q(\pi)\log\frac{p(f|f_Z,\pi)p(f_Z|f_{Z_m})q(f_{Z_m})q(\pi)p(y)}{p(y|f)p(f|f_Z,\pi)p(f_Z|f_{Z_m})p(f_{Z_m})p(\pi)}df df_Z df_{Z_m} d\pi$$

$$=\int p(f|f_Z,\pi)p(f_Z|f_{Z_m})q(f_{Z_m})q(\pi)\log\frac{\cancel{p(f|f_Z,\pi)}\cancel{p(f_Z|f_{Z_m})}q(f_{Z_m})q(\pi)p(y)}{p(y|f)\cancel{p(f|f_Z,\pi)}\cancel{p(f_Z|f_{Z_m})}p(f_{Z_m})p(\pi)}df df_Z df_{Z_m} d\pi$$

$$=\int p(f|f_Z,\pi)p(f_Z|f_{Z_m})q(f_{Z_m})q(\pi)\log\frac{q(f_{Z_m})q(\pi)p(y)}{p(y|f)p(f_{Z_m})p(\pi)}df df_Z df_{Z_m} d\pi$$\\

$$=\int p(f|f_Z,\pi)p(f_Z|f_{Z_m})q(f_{Z_m})q(\pi)\log\frac{q(f_{Z_m})}{p(f_{Z_m})}df df_Z df_{Z_m}d\pi$$
$$+ \int p(f|f_Z,\pi)p(f_Z|f_{Z_m})q(f_{Z_m})q(\pi)\log\frac{q(\pi)}{p(\pi)}df df_Z df_{Z_m}d\pi$$

$$-\int p(f|f_Z,\pi)p(f_Z|f_{Z_m})q(f_{Z_m})q(\pi)\log p(y|f) df df_Z df_{Z_m}d\pi$$
$$+\int p(f|f_Z,\pi)p(f_Z|f_{Z_m})q(f_{Z_m})q(\pi)\log p(y) df df_Z df_{Z_m}d\pi$$\\

$$=KL(q(f_{Z_m})||p(f_{Z_m}))+KL(q(\pi)||p(\pi))-\mathbb{E}_{q(f)}\left[\log p(y|f))\right]+\log p(y)$$

$$\stackrel{1)}{=}\sum_{i=1}^k KL(q(f_{i_m})||p(f_{i_m}))+KL(q(\pi)||p(\pi))-\mathbb{E}_{q(f)}\left[\log p(y|f))\right]+\log p(y)$$

$$\stackrel{2)}{=}\sum_{i=1}^k KL(q(f_{i_m})||p(f_{i_m}))+KL(q(\pi)||p(\pi))-\sum_{j=1}^n\left\{\sum_{i=1}^k\tilde{\pi}_i\mathcal{N}(y_j|k_{i,j}^T K_{i,mm}^{-1}m_{f_i},\sigma^2)-\frac{1}{2\sigma^2}\tilde{k}_{i,(j,j)}-\frac{1}{2}tr\left(S_{f_i}\Lambda_{i,j}\right)\right\}$$
$$+\log p(y)$$

and the ELBO is obtained as 

$$\log p(y)\geq\sum_{j=1}^n\left\{\sum_{i=1}^k \tilde{\pi}_i\mathcal{N}(y_j|k_{i,j}^T K_{i,mm}^{-1}m_{f_i},\sigma^2)-\frac{1}{2\sigma^2}\tilde{k}_{i,(j,j)}-\frac{1}{2}tr\left(S_{f_i}\Lambda_{i,j}\right)\right\}$$
$$- \sum_{i=1}^k KL(q(f_{i_m})||p(f_{i_m}))\stackrel{3}{-}KL(q(\pi)||p(\pi))$$\\

$1)$ follows from $q(f_{Z_m})=\prod_{i=1}^k\mathcal{N}(f_{i_m}|m_{f_i},S_{f_i})),\, p(f_{Z_m})=\prod_{i=1}^k\mathcal{N}(f_{i_m}|m_{i,m},K_{i,mm})))$\\

$2)$ follows from

$$\mathbb{E}_{q(f)}\left[\log p(y|f))\right]=\int q(f) \log p(y|f) df$$
$$=\int \sum_{i=1}^k\tilde{\pi}_i \mathcal{N}(f|K_{i,nm}K_{i,mm}^{-1}m_{f_i},K_{i,nn}-K_{i,nm}K_{i,mm}^{-1}(K_{i,mm}-S_{f_i})K_{i,mm}^{-1}K_{i,mn})\log p(y|f)df$$

$$\sum_{i=1}^k\tilde{\pi}_i\int  \mathcal{N}(f|K_{i,nm}K_{i,mm}^{-1}m_{f_i},K_{i,nn}-K_{i,nm}K_{i,mm}^{-1}(K_{i,mm}-S_{f_i})K_{i,mm}^{-1}K_{i,mn})\log p(y|f)df$$

with the $k$ integration terms matching the standard SVGP integration term as in \cite{hensmanbig} and thus simplified in the same manner.\\

$3)$ $KL(q(\pi)||p(\pi))=KL(Dir(\pi|\tilde{\alpha})||Dir(\pi|\tilde{\alpha}))$
$$=\log\Gamma(\tilde{\alpha}_0)-\sum_{i=1}^k\log\Gamma(\tilde{\alpha}_i)-\log\Gamma(\alpha_0)+\sum_{i=1}^k\Gamma(\alpha_i)+\sum_{i=1}^k(\tilde{\alpha}_i-\alpha_i)(\psi(\tilde{\alpha}_i)-\psi(\tilde{\alpha}_0))$$

with $\tilde{\alpha}_0=\sum_{i=1}^k \tilde{\alpha}_i,\, \alpha_0=\sum_{i=1}^k \alpha_i$; $\Gamma(\cdot)$ the gamma function, $\psi(\cdot)$ the digamma function. A proof can be found in \cite{dirichletkl}. The derivatives of the (log-) gamma and digamma functions can be calculated via the polygamma function and are implemented in most modern libraries for automatic differentiation.
\end{appendix}

\end{document}